\documentclass{article}

\usepackage{microtype}
\usepackage{graphicx}
\usepackage{subfigure}
\usepackage{booktabs}

\usepackage{color}
\usepackage{multirow}
\usepackage{threeparttable}
\usepackage{verbatim}

\usepackage{amsmath}
\usepackage{amsthm}
\newtheorem{theorem}{Theorem}

\usepackage{makecell}

\usepackage[accepted]{icml2021}

\icmltitlerunning{FTSO: Effective NAS via First Topology Second Operator}

\begin{document}

\twocolumn[
\icmltitle{FTSO: Effective NAS via First Topology Second Operator}

\begin{icmlauthorlist}

\icmlauthor{Likang Wang}{ust}
\icmlauthor{Lei Chen}{ust}
\end{icmlauthorlist}

\icmlaffiliation{ust}{Department of Computer Science and Technology, The Hong Kong University of Science and Technology, Hong Kong}

\icmlcorrespondingauthor{Likang Wang}{lwangcg@connect.ust.hk}
\icmlcorrespondingauthor{lei Chen}{leichen@cse.ust.hk}

\icmlkeywords{Machine Learning, ICML, Neural Architecture Search, NAS, DARTS, PC-DARTS, Computer Vision, Image Classification}

\vskip 0.3in
]

\printAffiliationsAndNotice{}  

\begin{abstract}
Existing one-shot neural architecture search (NAS) methods have to conduct a search over a giant super-net, which leads to the huge computational cost. To reduce such cost, in this paper, we propose a method, called FTSO, to divide the whole architecture search into two sub-steps. Specifically, in the first step, we only search for the topology, and in the second step, we search for the operators. FTSO not only reduces NAS’s search time from days to $0.68$ seconds, but also significantly improves the found architecture's accuracy. Our extensive experiments on ImageNet show that within $18$ seconds, FTSO can achieve a $76.4\%$ testing accuracy, $1.5\%$ higher than the SOTA, PC-DARTS. In addition, FTSO can reach a $97.77\%$ testing accuracy, $0.27\%$ higher than the SOTA, with nearly $100\%$ ($99.8\%$) search time saved, when searching on CIFAR10.
\end{abstract}

\section{Introduction}
Since the great success of the AlexNet~\citep{AlexNet} in image classification, most modern machine learning models~\cite{wang2023Dionysus,is-mvsnet,wang2023flora} have been developed based on deep neural networks. For neural networks, their performance is greatly determined by the architectures. Thus, in the past decade, a tremendous amount of work~\citep{VGGNet,GoogleNet,ResNet} has been done to investigate proper network architecture design. However, as the network size has grown larger and larger, it has gradually become unaffordable to manually search for better network architectures due to the expensive time and resource overheads. To ease this problem, a new technique called neural architecture search (NAS) was introduced. It allows computers to search for better network architectures automatically instead of relying on human experts.

Early-proposed reinforcement learning-based NAS methods~\citep{NASRL,DNNARL,NASNet} typically have an RNN-based controller to sample candidate network architectures from the search space. Although these algorithms can provide promising accuracy, their computation cost is usually unaffordable, for instance, 1800 GPU-days are required for NASNet to find an image classification network on CIFAR10.

To ease the search efficiency problem, one-shot approaches~\citep{ENAS,ProxylessNAS,DARTS} with parameter sharing have been proposed. These methods first create a huge directed acyclic graph (DAG) super-net, containing the whole search space. Then, the kernel weights are shared among all the sampled architectures via the super-net. This strategy makes it possible to measure the candidate architecture's performance without repeatedly retraining it from scratch. However, these algorithms suffer from the super-nets' computational overheads. This problem is particularly severe for differentiable models~\citep{DARTS,PC-DARTS}.

Limited by current NAS algorithms' inefficiency, it is rather challenging to find satisfying network architectures on large-scale datasets and complex tasks. For instance, current speed-oriented NAS approaches generally require days to accomplish one search trial on ImageNet, for example, $8.3$ GPU-days for ProxylessNAS~\citep{ProxylessNAS} and $3.8$ GPU-days for PC-DARTS~\citep{PC-DARTS}. Therefore, we argue that it is essential to propose a new well-defined search space, which is not only expressive enough to cover the most powerful architectures, but also compact enough to filter out the poor architectures.

\begin{figure*}[htbp]
\centering
\includegraphics[width=0.75\textwidth]{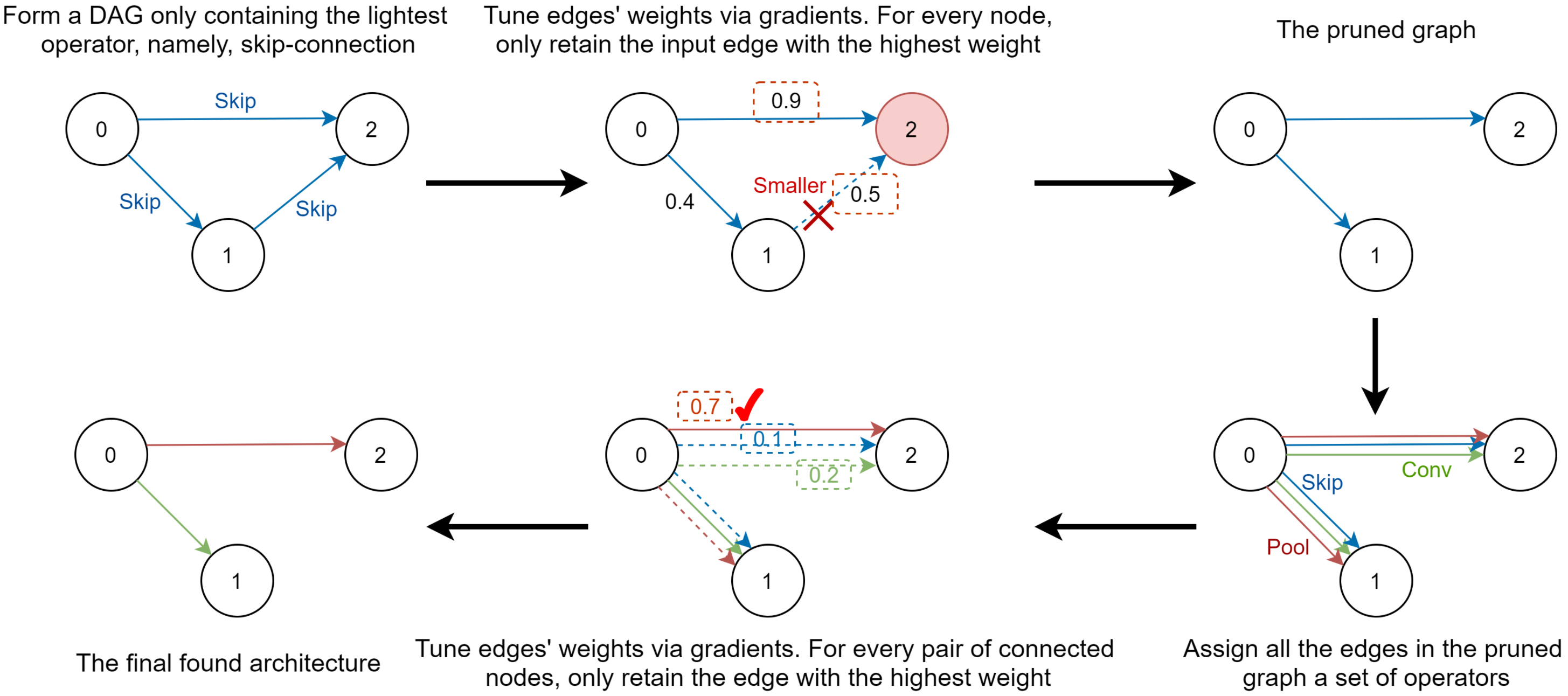}
\vspace{-1ex}
\caption{The main structure of FTSO.}
\vspace{-2ex}
\label{FTSO_Structure}
\end{figure*}

We are motivated by \citet{Understanding_Cell}, who demonstrate that randomly replacing operators in a found architecture does not harm the accuracy much. As such, we believe that there would be no reduction in the test accuracy if we omit the influence of operators and cluster architectures according to the topology. Thus, in this paper, we propose to separately search for the network topology and the operators. We name this new method Effective NAS via First Topology Second Operator (FTSO). 

In this paper, we first mathematically prove that FTSO reduces the number of network parameters by $10^{8}$, decreases the FLOPs per iteration by $10^{5}$ and lowers the operator's complexity in magnitude. We then empirically reveal that FTSO shortens the required search period from 50 epochs to one iteration. Besides the great improvement in efficiency, FTSO also significantly promotes the effectiveness by easing the over-fitting phenomenon and the Matthew effect \citep{Understanding_Cell}. To be specific, each architecture in DARTS has only one iteration to tune its kernel weights, and within one iteration, only the operators with few parameters may converge. The result is that the simpler operators outperform the more powerful ones in the super-net, then larger gradients to enlarge their advantages are achieved. In this way, the found architectures tend to only contain the simplest operators and perform poorly on both the training and testing sets. Such phenomenon is called the Matthew effect.

\begin{algorithm}[htbp]
\caption{topology search}
\label{alg:algorithm_topology_search} 
\begin{algorithmic}
\REQUIRE a set of nodes: $n_k$
\ENSURE the pruned architecture: $A_p$
\STATE 1. Create an directed edge $e_{i,j}$ with weight $\beta_{i,j}$ between each pair of nodes $n_i$ and $n_j$ ($i<j$) \par
\STATE 2. Assign each edge $e_{i,j}$ a \textit{skip connection} operator $o_{i,j}$ with kernel weights $w_{i,j}$ \par
\WHILE{still in the first epoch}
    \STATE 1. Forward-propagate following $n_j=\sum_{i<j}o(n_i)\beta_{i,j}$\par
    \STATE 2. Update architecture $\beta$ by descending $\nabla_\beta \mathcal{L}_{val}(w,\beta)$\par
    \STATE 3. Update weights $w$ by descending $\nabla_w \mathcal{L}_{train}(w,\beta)$\par
\ENDWHILE
\FOR{each node $n_j \in A_p$}
    \STATE $T_j$ $\leftarrow$ the second largest $\beta_{i,j}$\par
    \FOR{each node $n_i$}
        \IF{$\beta_{i,j}<T_j$}
            \STATE Prune edge $e_{i,j}$
        \ENDIF
    \ENDFOR
\ENDFOR
\STATE Derive the pruned architecture $A_p$.
\end{algorithmic}
\end{algorithm}

Our extensive experiments show that FTSO can accomplish the whole architecture search in $0.68$ seconds. On ImageNet, FTSO achieves $76.4\%$ testing accuracy, $1.5\%$ higher than the SOTA, within a mere 18 seconds. More importantly, when we only search for one iteration, FTSO consumes less than $0.68$ seconds, while reaching $75.64\%$ testing accuracy, $0.74\%$ higher than the SOTA. Moreover, if we allow FTSO to search for $19$ minutes, $76.42\%$ Top1 and $93.2\%$ Top5 testing accuracy can be achieved. In addition, FTSO can reach $97.77\%$ testing accuracy, $0.27\%$ higher than the SOTA, with nearly $100\%$ ($99.8\%$) search time saved, when searching on CIFAR10. Although in this paper we have implemented FTSO within a continuous search space, we illustrate in Section \ref{section: Generalization to other tasks and search spaces} that FTSO can be seamlessly transferred to other NAS algorithms.

\begin{algorithm}[htbp]
\caption{operator search}
\label{alg:algorithm_operator_search}
\begin{algorithmic}
\REQUIRE{the pruned architecture produced by the topology search: $A_p$} 
\ENSURE{the found architecture: $A_f$} 
\IF{replace with convolutions}
    \STATE Replace all the retained operators $o_{i,j}$ in $A_p$ with convolutions\par
\ELSE
    \STATE Each node $n_j \leftarrow \sum_{i<j}\sum_{o\in\mathcal{O}}\frac{\exp{\alpha^o_{i,j}}}{\sum_{o'\in\mathcal{O}}\exp{\alpha^{o'}_{i,j}}}o(n_i)$\par
    \WHILE{not converged}
        \STATE Update architecture $\alpha$ by descending $\nabla_\alpha \mathcal{L}_{val}(w,\alpha)$\par
        \STATE Update weights $w$ by descending $\nabla_w \mathcal{L}_{train}(w,\alpha)$\par
    \ENDWHILE
\ENDIF
\FOR{each edge $e_{i,j} \in A_p$}
    \STATE Assign edge $e_{i,j}$ the operator $o'\in\mathcal{O}$ with the highest $\alpha^{o'}_{i,j}$\par
\ENDFOR
\STATE Derive the found architecture $A_f$ $\leftarrow$ $A_p$.
\end{algorithmic}
\end{algorithm}
\vspace{-2ex}

\begin{figure*}[htbp]
\begin{center}
\vspace{-2ex}
\subfigure[]{
		\begin{minipage}[b]{0.21\textwidth}
			\includegraphics[width=1\textwidth]{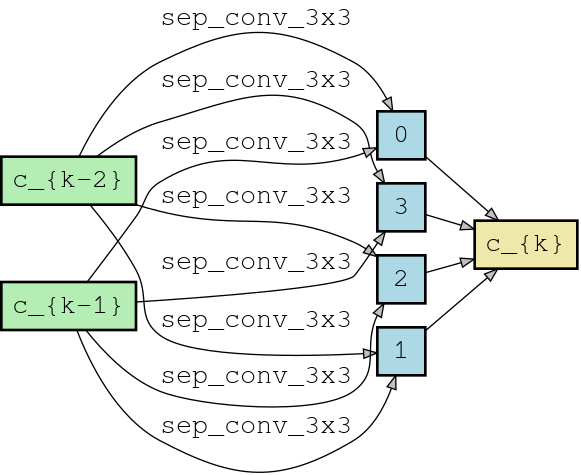} 
		\end{minipage}
		\label{fig:Normal Cell CIFAR10}
	}
    \subfigure[]{
		\begin{minipage}[b]{0.44\textwidth}
			\includegraphics[width=1\textwidth]{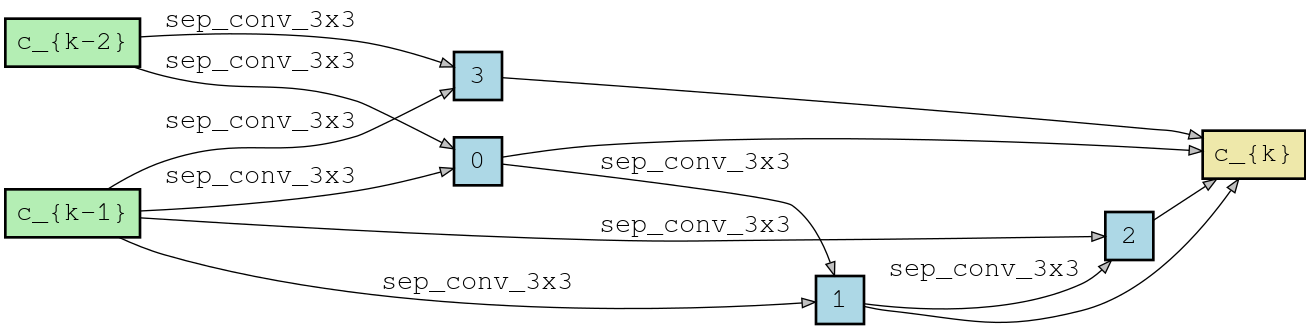} 
		\end{minipage}
		\label{fig:Reduction Cell CIFAR10}
	}
\subfigure[]{
		\begin{minipage}[b]{0.21\textwidth}
			\includegraphics[width=1\textwidth]{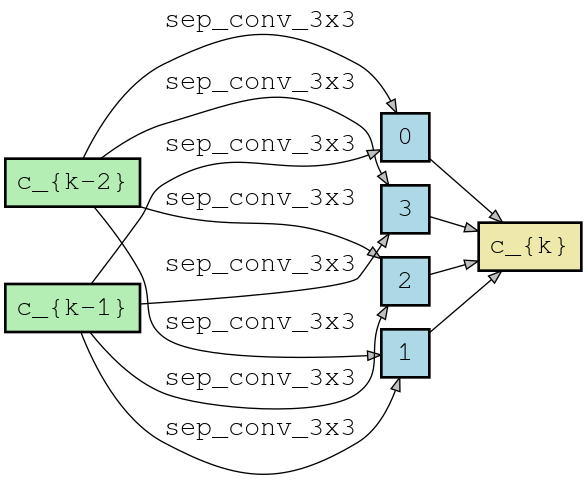} 
		\end{minipage}
		\label{fig:Normal Cell Full ImageNet}
		}
	\subfigure[]{
	    \begin{minipage}[b]{0.18\textwidth}
		    \includegraphics[width=1\textwidth]{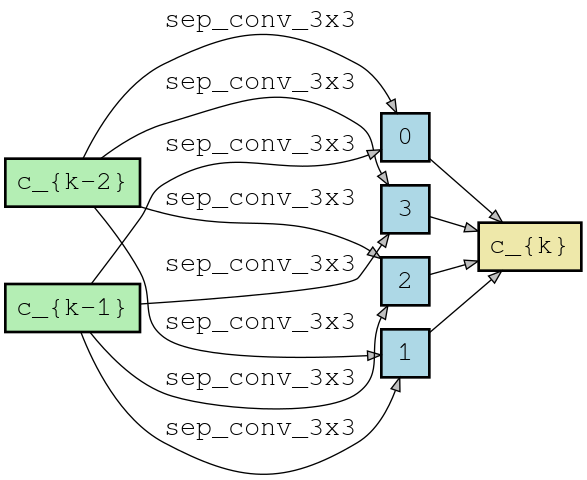} 
	    \end{minipage}
    	\label{fig:Reduction Cell Full ImageNet}
	}
\subfigure[]{
		\begin{minipage}[b]{0.41\textwidth}
			\includegraphics[width=1\textwidth]{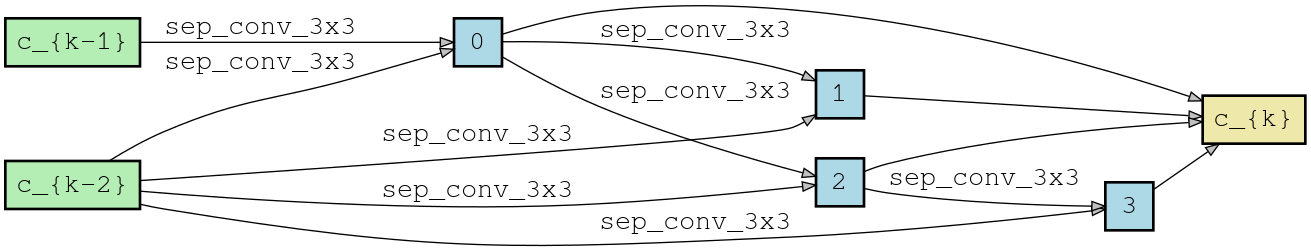} 
		\end{minipage}
		\label{fig:Normal Cell CIFAR10 sep + op}
	}
    \subfigure[]{
		\begin{minipage}[b]{0.36\textwidth}
			\includegraphics[width=1\textwidth]{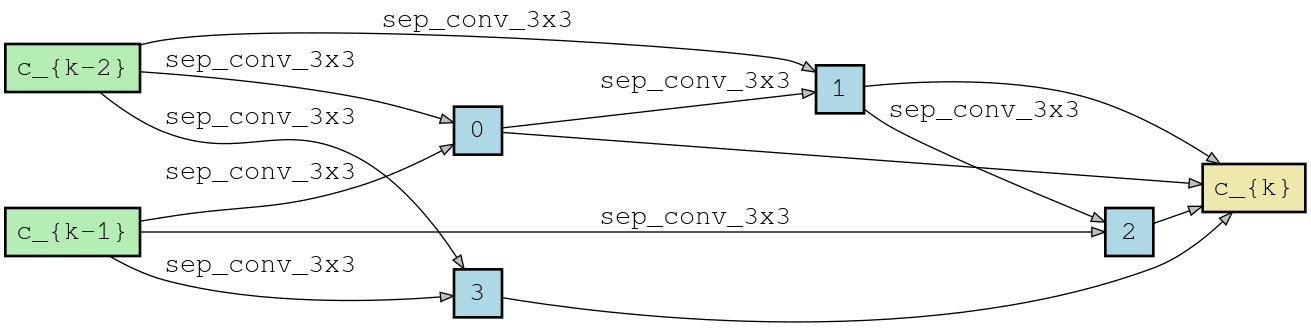} 
		\end{minipage}
		\label{fig:Reduction Cell CIFAR10 sep + op}
	}	
\subfigure[]{
		\begin{minipage}[b]{0.45\textwidth}
			\includegraphics[width=1\textwidth]{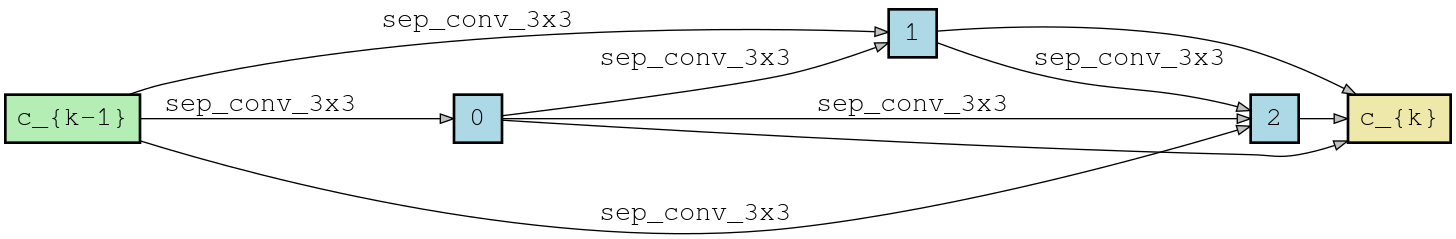} 
		\end{minipage}
		\label{fig:NASBenchallsep}
	}
    \subfigure[]{
		\begin{minipage}[b]{0.45\textwidth}
			\includegraphics[width=1\textwidth]{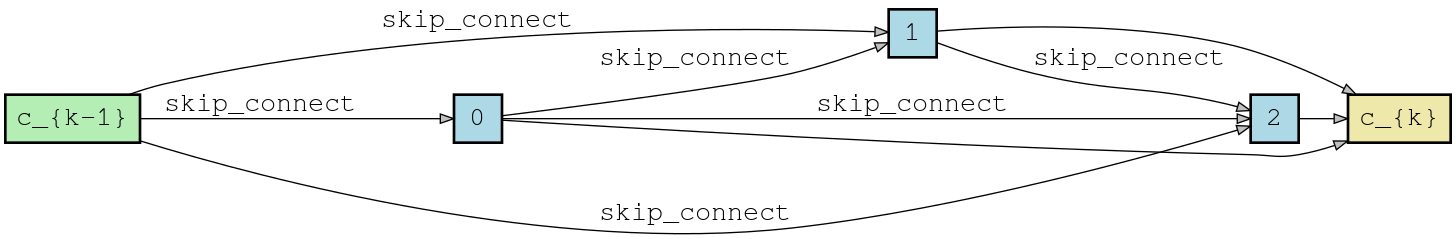} 
		\end{minipage}
		\label{fig:NASBenchallskip}
	}	
\vspace{-2ex}
\caption{FTSO's found architectures.    \subref{fig:Normal Cell CIFAR10} and \subref{fig:Reduction Cell CIFAR10}: Normal and reduction cells found on CIFAR10 after one epoch’s search;   \subref{fig:Normal Cell Full ImageNet} and \subref{fig:Reduction Cell Full ImageNet}: Normal and reduction cells found on the entire ImageNet after one epoch’s search;   \subref{fig:Normal Cell CIFAR10 sep + op} and \subref{fig:Reduction Cell CIFAR10 sep + op}: Normal and reduction cells found on CIFAR10, where we adopt the operator search, and use the \textit{$3\times3$ separable convolution} to search for the topology;   \subref{fig:NASBenchallsep}: FTSO's cell found on NATS-Bench; \subref{fig:NASBenchallskip}: DARTS's cell found on NATS-Bench.}
\vspace{-2ex}
\label{FTSO's found architecture CIFAR10}
\end{center}
\end{figure*}

\section{Related Work}
In general, existing NAS algorithms can be divided into three categories, namely, reinforcement learning-based, revolution-based and differentiable. Early-proposed reinforcement learning-based methods~\citep{NASRL, NASNet} generally suffer from high computational cost and low-efficiency sampling. Instead of sampling a discrete architecture and then evaluating it, DARTS~\citep{DARTS} treats the whole search space as a continuous super-net. It assigns every operator a real number weight and treats every node as the linear combination of all its transformed predecessors. To be specific, DARTS's search space is a directed acyclic graph (DAG) containing two input nodes inherited from previous cells, four intermediate nodes and one output node. Each node denotes one latent representation and each edge denotes an operator. Every intermediate node $\mathbf{x}_j$ is calculated from all its predecessors $\mathbf{x}_i$, i.e., $\mathbf{x}_j=\sum_{i<j}\sum_{o\in\mathcal{O}}\frac{\exp{\alpha^o_{i,j}}}{\sum_{o'\in\mathcal{O}}\exp{\alpha^{o'}_{i,j}}}o(\mathbf{x}_i)$, where $\mathcal{O}$ denotes the collection of all candidate operators, $\alpha^o_{i,j}$ denotes the weight for operator $o$ from node $i$ to $j$. This strategy allows DARTS to directly use gradients to optimize the whole super-net. After the super-net converges, DARTS only retains the operators with the largest weights. In this way, the final discrete architecture is derived. The main defect of DARTS is that it needs to maintain and do all calculations on a giant super-net, which inevitably leads to heavy computational overheads and over-fitting.

To relieve the computational overhead of DARTS, DARTS-ES~\citep{DARTS-ES} reduces the number of searching epochs via early stopping, according to the Hessian matrix's max eigenvalue. PC-DARTS~\citep{PC-DARTS} decreases the FLOPs per iteration by only calculating a proportion of the input channels and keeping the remainder unchanged, and normalizes the edge weights to stabilize the search. To be specific, in PC-DARTS, every intermediate node $\mathbf{x}_j$ is computed from all its predecessors $\mathbf{x}_i$, i.e., $\mathbf{x}_j=\sum_{i<j}\frac{\exp{\beta_{i,j}}}{\sum_{i'<j}\exp{\beta_{i',j}}}f_{i,j}(\mathbf{x}_i)$, where $\beta_{i,j}$ describes the input node $i$'s importance to the node $j$, and $f_{i,j}$ is the weighted sum of all the candidate operators' outputs between node $i$ and $j$. Specifically, $f_{i,j}(\mathbf{x}_i,\mathbf{S}_{i,j})=\sum_{o\in{\mathcal{O}}}\frac{e^{\alpha^o_{i,j}}}{\sum_{o'\in{\mathcal{O}}}e^{\alpha^{o'}_{i,j}}}o(\mathbf{S}_{i,j}*\mathbf{x}_i)+(1-\mathbf{S}_{i,j})*\mathbf{x}_i$, where $\mathbf{S}_{i,j}$ denotes a binary vector, in which only $1/K$ elements are $1$.

\section{FTSO: Effective NAS via First Topology Second Operator}
\label{section: FTSO: Effective NAS via First Topology Second Operator}
\label{ssec:topologySearch}

Existing NAS approaches generally suffer from a huge computational overhead and an unsatisfying testing accuracy led by the huge search space. Such problems are especially stern in one-shot and differentiable methods because these algorithms need to maintain and even do all the calculations directly on the search space.

\begin{table*}[htbp]
\vspace{-2ex}
\caption{Different configurations' impacts to FTSO}
\label{Table:: Ablation study on FTSO}
\begin{center}
\begin{threeparttable}
\begin{tabular}{lcccccc}
\toprule
\multirow{2}*{\bf{FTSO's Configuration}}  &
\multicolumn{2}{c}{\bf{CIFAR10 Error (\%)}} &
\multicolumn{2}{c}{\bf{ImageNet Error (\%)}} &
\multicolumn{2}{c}{\bf{Search Cost (GPU-days)}}\\
 \cmidrule(lr){2-3} \cmidrule(lr){4-5} \cmidrule(lr){6-7}
 & \bf{600 Epoch} & \bf{1200 Epoch} & \bf{Top1} & \bf{Top5} & \bf{CIFAR} & \bf{ImageNet} \\

 \midrule
CIF. Topo(skip,1it.)         &$2.68$         &$2.54$        &$24.36$  &$7.27$  
&$7.87\times10^{-6}$   &-  \\
CIF. Topo(skip,1ep.)         &$2.48$         &$2.23$   &$23.60$  &$7.01$
&$2\times10^{-4}$  &-   \\
CIF. Topo(skip,50ep.)  &2.77   &2.52   &-  &- &$0.01$  &-   \\
CIF. Topo(skip,1ep.)+Op(18ep.)    &$2.85$ &$2.59$  &$23.97$         &$7.20$          &$0.01$  &-      \\
CIF. Topo(skip,50ep.)+Op(50ep.)    &$2.59$  &$2.36$   &$23.97$         &$7.12$       &$0.05$  &- \\
CIF. Topo(m.p.,50ep.)+Op(50ep.)   &$2.83$  &$2.48$    &-         &-    &0.05   &- \\
CIF. Topo(sep3,50ep.)         &2.63         &2.48   &- &-      &$0.02$     &-    \\
CIF. Topo(sep3,50ep.)+Op(50ep.)   &$2.56$  &$2.52$    &$24.73$         &$7.60$    &$0.06$   &- \\
CIF. Topo(3op,50ep.)+Op(50ep.)   &$2.59$  &$2.50$    &-         &-    &$0.05$   &- \\
CIF. Topo(4op,50ep.)+Op(50ep.)   &$2.68$  &$2.59$    &-     &-    &$0.07$    &- \\
Part ImageNet Topo(skip,1it.)    &- &-     &$24.03$         &$7.07$         &-  &$0.0002$     \\
Part ImageNet Topo(skip,1ep.)    &- &-     &$23.94$         &$7.05$         &-  &$0.0017$     \\
Part ImageNet Topo(skip,6ep.)    &- &-     &$24.59$         &$7.38$       &-   &$0.009$    \\
Full ImageNet Topo(skip,1ep.)    &$2.35$ &$2.26$     &$23.58$         &$6.80$     &-    &$0.01$ \\
\bottomrule
\end{tabular}
\begin{tablenotes}
		\item[*] '3op' means \textit{max pool 3x3}, \textit{skip connect} and \textit{none};     '4op' means: \textit{sep conv 3x3}, \textit{max pool 3x3}, \textit{skip connect} and \textit{none};      'm.p.' means: \textit{max pool 3x3};      'sep3' means: \textit{sep conv 3x3}.      
		\item[*] 'CIF.' means CIFAR10;      'Topo(skip,1it)' means to search for topology with only skip connections for 1 iteration;   '1ep' means 1 epoch;    'Part ImageNet' means to search on part of ImageNet.
     \end{tablenotes}
\end{threeparttable}
\end{center}
\vspace{-3ex}
\end{table*}

To ease such problems, it is of great demand to investigate the correlations among different architectures and to shrink the search space according to the prior knowledge. We notice that there is an important observation in \citet{Understanding_Cell} that randomly substituting the operators in a found architecture does not observably influence the testing accuracy. Therefore, it would be great inspiration if we could cluster the architectures according to their connection topologies. To be specific, suppose we find an architecture only containing the simplest operators achieves high accuracy on the testing set, if we replace all the connections in this architecture with powerful operators, the converted architecture can perform well on the testing set with high confidence. 

In this paper, we first propose to find the most effective network topology with simple operators. We then fix the topology, and search for the most suitable operators for the given topology. In this way, the testing accuracy can still be guaranteed, while the search space is shrunk in magnitude. We name this new NAS algorithm Effective NAS via First Topology Second Operator (FTSO).

\begin{table}[htbp]
\vspace{-3ex}
\caption{Symbol table}
\label{symbol table}
\begin{center}
\begin{threeparttable}
\begin{tabular}{ll}
\toprule
\bf{Symbol} & \bf{Meaning}\\
\midrule
$\alpha^o_{i,j}$ & Weight for operator $o$ from node $i$ to $j$ \\
$\beta_{i,j}$ & Node $i$'s importance to the node $j$ \\
\midrule
$f_{i,j}$ & \makecell[l]{Linear combination of all the candidate\\ operators $o$ with weight $\alpha^o_{i,j}$} \\
\bottomrule
\end{tabular}
\end{threeparttable}
\end{center}
\vspace{-2ex}
\end{table}

We summarize the symbols used in this section in Table \ref{symbol table}. As shown in Figure~\ref{FTSO_Structure}, we inherit the differentiable framework of PC-DARTS, and divide the architecture search into two phases. We name the two phases topology search and operator search, and illustrate how they work in Algorithms \ref{alg:algorithm_topology_search} and \ref{alg:algorithm_operator_search}, respectively. In the first phase, we form a super-net only containing the simplest operator, \textit{skip connection}. Since the \textit{skip connection} operator contains no kernel weights, we only need to optimize the architecture parameters $\beta_{i,j}$.
In fact, as shown in Table \ref{Table:: Ablation study on FTSO}, the \textit{max pooling} operator also delivers satisfying results for the topology search. There are two reasons we use \textit{skip connection}.
The first is that the \textit{skip connection} operator not only requires zero parameters, but also demands the minimum computational cost. The second reason is that \textit{max pooling} may lead to the loss of useful information if the network is deep. 
Furthermore, as the only difference between our topology search and the vanilla DARTS is the number of candidate operators, the pruned architecture's connectivity can be guaranteed.

Similar to DARTS and PC-DARTS, after the topology search, for every intermediate node $j$, we only retain its connections to its predecessors $i^*$ with the highest two $\beta_{i,j}$. In the second phase, we search for the operators suitable for the pruned topology with two strategies. The first strategy is similar to the vanilla DARTS. It replaces each operator in the pruned topology with a mix-operator $f_{i,j}$.
After that, we optimize the architecture parameters, $\alpha^o_{i,j}$, $\beta_{i,j}$ and the kernel weights $\omega_{i,j}^o$ alternatively. After the super-net converges,  we only retain one operator $o^*$ with the highest $\alpha^o_{i,j}$ for every two connected nodes $i$ and $j$. The second strategy directly replaces all the operators in the pruned topology with one single operator owning the highest model capacity, e.g., a convolution operator. 

\begin{table*}[htbp]
\caption{Comparison with existing state-of-the-art image classification architectures}
\label{Comp_CIFAR10}
\begin{center}
\begin{threeparttable}
\begin{tabular}{lcccccc}
\toprule
\multirow{2}*{\bf{Architecture}}  & 
\multicolumn{2}{c}{\bf{CIFAR Err. (\%)}} &
\multicolumn{2}{c}{\bf{ImageNet Err. (\%)}} &
\multicolumn{2}{c}{\bf{Search Cost (GPU-days)}}\\
 \cmidrule(lr){2-3} \cmidrule(lr){4-5} \cmidrule(lr){6-7}
 & \bf{600 Ep.} & \bf{1200 Ep.} & \bf{Top1} & \bf{Top5} & \bf{CIFAR} & \bf{ImageNet} \\
 \midrule
NASNet-A$^{\dagger}$\citep{NASNet}         &$2.65$         &-         &$26.0$  &$8.4$  &$1800$  &-\\
AmoebaNet$^{\dagger}$\citep{AmoebaNet}         &$2.55$         &-         &$24.3$  &$7.6$  &$3150$  &- \\
PNAS\citep{PNAS}         &$3.41$         &-         &$25.8$  &$8.1$  &$225$ &- \\
ENAS$^{\dagger}$\citep{ENAS} &$2.89$  &-  &-  &-  &$0.5$ &-\\
 \midrule
DARTS (2nd)$^{\dagger}$\citep{DARTS}         &$2.76$         &-   &$26.7$  &$8.7$      &1   &4.0     \\
SNAS$^{\dagger}$\citep{SNAS}         &$2.85$         &-    &$27.3$  &$9.2$     &1.5   \\
ProxylessNAS$^{\dagger*}$\citep{ProxylessNAS}         &$2.08$         &-    &$24.9$  &$7.5$     &$4.0$   &$8.3$      \\
P-DARTS$^{\dagger}$\citep{P-DARTS}         &$2.50$         &-     &$24.4$  &$7.4$   &0.3 &-         \\
BayesNAS$^{\dagger}$\citep{BayesNAS}         &$2.81$         &-  &$26.5$  &$8.9$     &0.2   &-         \\
 \midrule
PC-DARTS(CIFAR10)$^{\dagger}$\citep{PC-DARTS}         &$2.57$         &$2.50$    &$25.1$  &$7.8$     &$0.1$  &-       \\
PC-DARTS(ImageNet)$^{\dagger*}$\citep{PC-DARTS}         &-         &-    &$24.2$  &$7.3$     &- &$3.8$         \\
 \midrule
FTSO (CIFAR10 + 1 epoch)$^{\dagger}$         &2.48         &\bf{2.23}   &$23.60$  &$7.01$      
&$2\times10^{-4}$  &-   \\
FTSO (CIFAR10 + 1 iteration)$^{\dagger}$         &2.68         &2.54        &$24.36$  &$7.27$  
&$\bf{7.87\times10^{-6}}$   &-  \\
FTSO (Full ImageNet + 1epoch)$^{\dagger*}$    &$\bf{2.35}$ &$2.26$     &$\bf{23.58}$         &$\bf{6.80}$     &-    &$\bf{0.01}$ \\
\bottomrule
\end{tabular}
\begin{tablenotes}
		\item $^{\dagger}$ When testing on CIFAR10, these models adopt cut-out.
		\item $^*$ These models are directly searched on ImageNet.
     \end{tablenotes}
\end{threeparttable}
\end{center}
\vspace{-3ex}
\end{table*}

In this paper, we take the second strategy as the default configuration because it is not only much more efficient, but also avoids the over-fitting phenomenon and the Matthew effect in DARTS. To be specific, in DARTS-like methods, suppose the found network perfectly fits the data, then the super-net must severely over-fit, since the super-net is much larger than the found network. As we know, an over-fitting model can hardly generalize well. Thus, the generated sub-graph is not likely to be the best architecture.
While in the second strategy, since no super-net is adopted and all the simple operators in the sub-graph are replaced with powerful operators, the final architecture’s model capacity gets promoted. Additional empirical comparisons between these two strategies can be found in Section \ref{section: ablation study}.

\begin{table*}[htbp]
\vspace{-2ex}
\caption{Comparison with existing state-of-the-art image classification architectures (NATS-Bench)}
\label{Comp_NATSBench}
\begin{center}
\begin{threeparttable}
\begin{tabular}{lccccccccc}
\toprule
\multirow{2}*{\bf{Architecture}}
& \multicolumn{3}{c}{\bf{Search on CIFAR10}} &
\multicolumn{3}{c}{\bf{Search on CIFAR100}} &
\multicolumn{3}{c}{\bf{Search on ImageNet}}\\
 \cmidrule(lr){2-4} \cmidrule(lr){5-7} \cmidrule(lr){8-10}
 & \bf{CF10$^*$} & \bf{CF100} & \bf{ImageNet} & \bf{CF10} & \bf{CF100} & \bf{ImageNet} & \bf{CF10} & \bf{CF100} & \bf{ImageNet} \\
\midrule
DARTS (1st) &$54.30$ &$15.61$ &$16.32^{\dagger}$ &$86.57$ &$58.13$ &$28.50$ &$89.58$ &$63.89$ &$33.77$\\
DARTS (2nd) &$86.88$ &$58.61$ &$28.91$ &$91.96$ &$67.27$ &$39.47$ &$84.64$ &$55.15$ &$26.06$\\
 \midrule
FTSO &$93.98$ &$70.22$ &$45.57$ &$93.98$ &$70.22$ &$45.57$ &$93.98$ &$70.22$ &$45.57$\\
\bottomrule
\end{tabular}
\begin{tablenotes}
		\item $^{\dagger}$ This means that within NATS-Bench's search space, when we use the 1st order DARTS to search for architectures on the CIFAR10 dataset, the found architecture can achieve $16.32\%$ testing accuracy on ImageNet.
		\item $^*$ CF10 means testing accuracy (\%) on CIFAR10; CF100 means testing accuracy (\%) on CIFAR100
     \end{tablenotes}
\end{threeparttable}
\end{center}
\vspace{-2ex}
\end{table*}

In DARTS, the network topology and operators are jointly searched, which makes both the size and the computational cost of the super-net extremely high. We use $n$ to denote the number of nodes and $p$ to denote the number of candidate operators. Since we have two input nodes, one output node and $n-3$ intermediate nodes, the super-net contains a total of $\frac{1}{2}(n^2-3n)$ edges. At the same time, every edge keeps $p$ operators, thus, the total number of operators in DARTS is $\frac{1}{2}(n^2-3n)p$.
By comparison, there are only $\frac{1}{2}n(n-3)$ operations in our topology search, and $2(n-3)p$ operations in our operator search. This is because in the topology search, every edge contains only one operator; and in the topology search, every intermediate node only connects to two predecessors. Since $n$ is usually close to $p$, FTSO reduces the number of operations from $O(n^3)$ to $O(n^2)$.

\begin{theorem}
\vspace{+1ex}
\label{theorem_darts_flops_paras}
The total number of FLOPs and parameters of DARTS are $\frac{1}{2}pn(n-3)H_{out}W_{out}C_{out}(k^2C_{in}+1)$ and $\frac{1}{2}n(n-3)p(k^2C_{in}+1)C_{out}$ respectively.
\end{theorem}
\begin{proof}
\vspace{-2ex}
Each vanilla convolutional operator needs $k^2C_{in}H_{out}W_{out}C_{out}$ FLOPs and $(k^2C_{in}+1)C_{out}$ parameters, where $k$ is the kernel size, $C_{in}$ is the input tensor's channel number and $H_{out}$, $W_{out}$ and $C_{out}$ are the output tensor's height, width and channel number respectively. For simplicity, assume all the candidate operators are convolutions. Since DARTS has $\frac{1}{2}pn(n-3)$ edges, it needs to compute $\frac{1}{2}pn(n-3)$ convolutions and $\frac{1}{2}pn(n-3)$ tensor summations. Owing to each tensor summation consuming $H_{in}W_{in}C_{in}$ FLOPs, DARTS requires a total of
$\frac{1}{2}pk^2n(n-3)C_{in}H_{out}W_{out}C_{out}$ convolution FLOPs and $\frac{1}{2}pn(n-3)H_{out}W_{out}C_{out}$ summation FLOPs. Thus, the overall FLOPs are parameters are $\frac{1}{2}pn(n-3)H_{out}W_{out}C_{out}(k^2C_{in}+1)$ and $\frac{1}{2}n(n-3)p(k^2C_{in}+1)C_{out}$, respectively.
\end{proof}

\begin{theorem}
\vspace{+1ex}
\label{theorem_ftso_flops_paras}
The total number of FLOPs and parameters of FTSO are $\frac{1}{2}n(n-3)H_{in}W_{in}C_{in}$ and $\frac{1}{2}n(n-3)$ respectively.
\end{theorem}
\begin{proof}
\vspace{-2ex}
Each \textit{skip connection} operator needs $0$ parameters and $0$ FLOPs. If we first search for the topology and then directly substitute the operators, only $\frac{1}{2}n(n-3)$ tensor summations need to be calculated, since FTSO has $\frac{1}{2}n(n-3)$ operators.
\end{proof}

In addition to the reduction in the number of operations, FTSO also dramatically decreases the internal cost of the operations. This is because during the topology search all the powerful operators are replaced by the simple operators.
We summarize the main properties of DARTS and FTSO in Theorems \ref{theorem_darts_flops_paras} and \ref{theorem_ftso_flops_paras}. As a typical configuration, let $k=5$, $C_{in}=C_{out}=512$, $n=7$, $p=8$. Then, our algorithm requires only $\frac{1}{p(k^2C_{in}+1)C_{out}}=1.9\times 10^{-8}$ times the parameters and $\frac{1}{p(k^2C_{in}+1)}=9.8\times 10^{-6}$ times the forward-propagation FLOPs per iteration compared to those of DARTS. 

\begin{figure}[htbp]
	\centering
	\vspace{-1ex}
	\subfigure[]{
		\begin{minipage}[b]{0.22\textwidth}
			\includegraphics[width=1\textwidth]{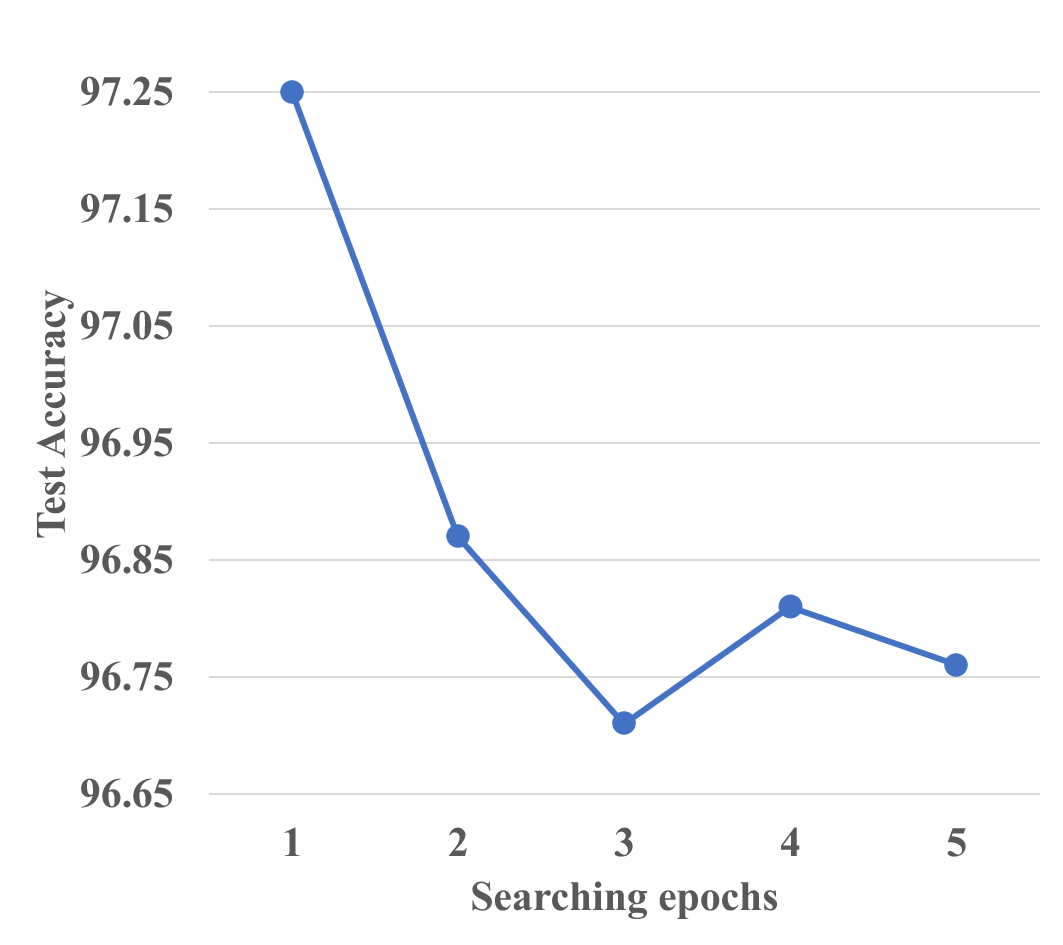} 
		\end{minipage}
		\label{fig:CIFAR10TestAccuracySearchEpochSkipConvinOneRun}
	}
    	\subfigure[]{
    		\begin{minipage}[b]{0.22\textwidth}
   		 	\includegraphics[width=1\textwidth]{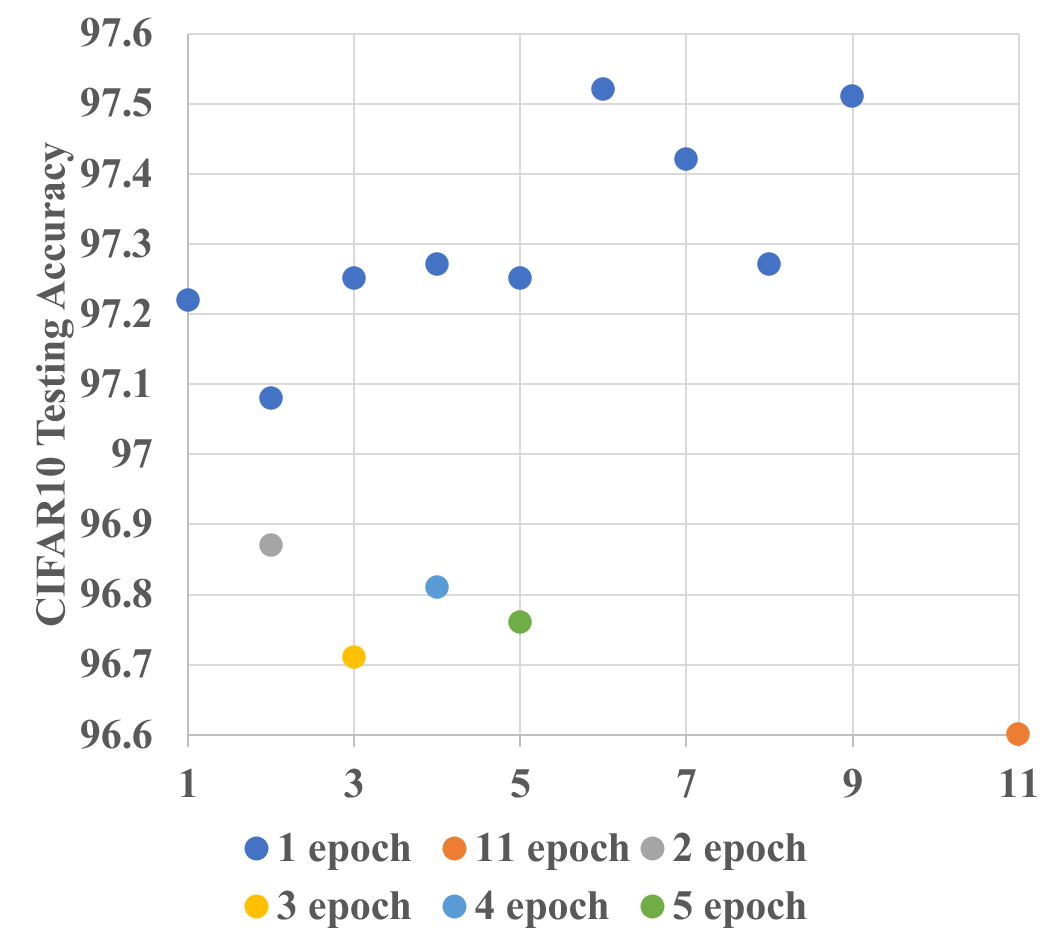}
    		\end{minipage}
		\label{fig:CIFAR10TestAccuracy-SearchEpochSkip-Conv}
    	}
	\subfigure[]{
		\begin{minipage}[b]{0.22\textwidth}
			\includegraphics[width=1\textwidth]{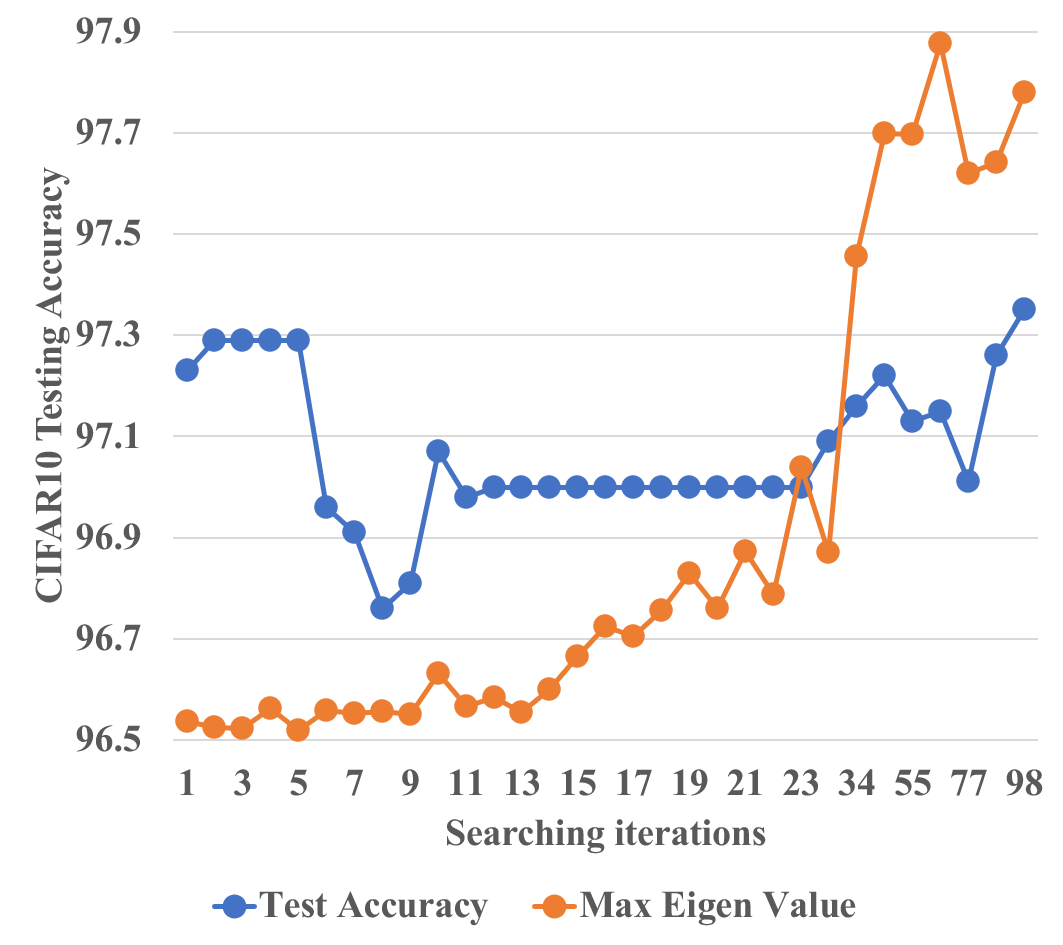} 
		\end{minipage}
		\label{fig:CIFAR10 Test Accuracy – Search Iteration (Skip-Conv)}
	}
    	\subfigure[]{
    		\begin{minipage}[b]{0.22\textwidth}
		 	\includegraphics[width=1\textwidth]{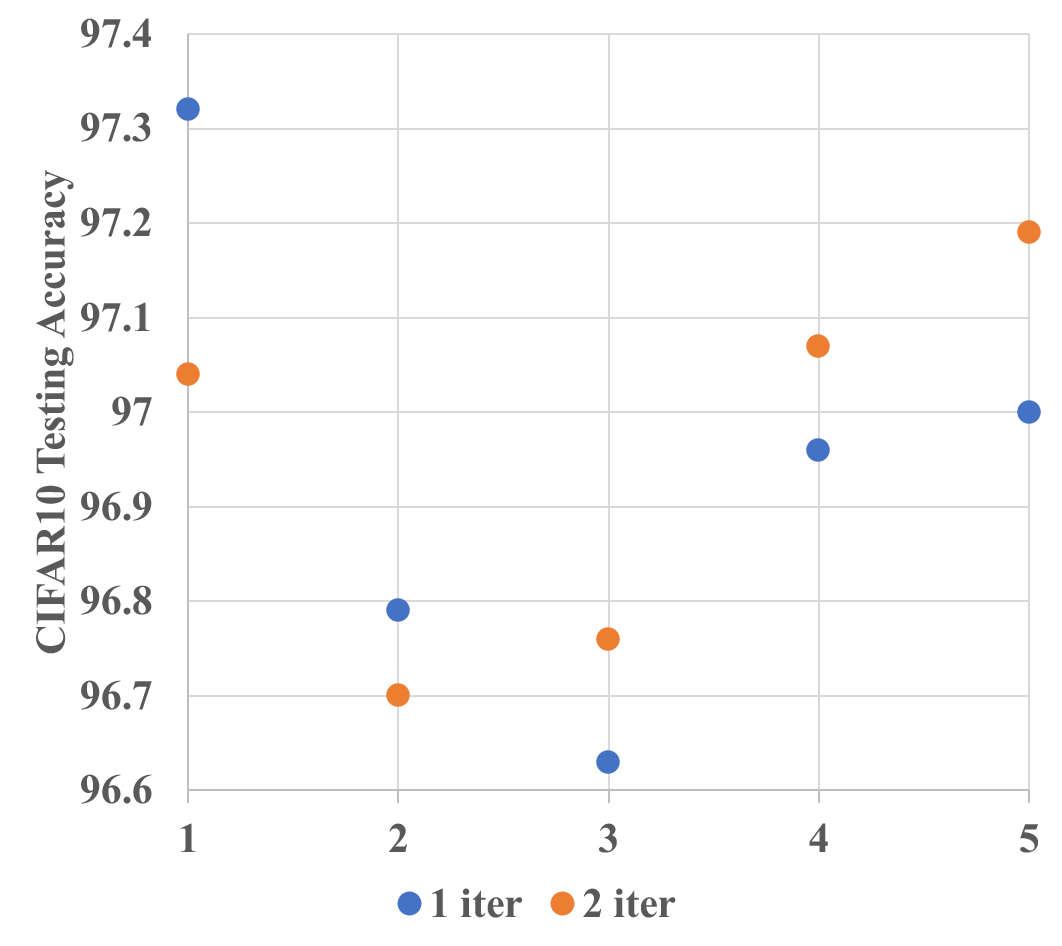}
    		\end{minipage}
		\label{fig:CIFAR10_1_vs_2_Iteration_Search_Accuracy_Comparison}
    	}
    	\vspace{-2ex}
	\caption{Ablation study Part 1.   \subref{fig:CIFAR10TestAccuracySearchEpochSkipConvinOneRun}: CIFAR10: Accuracy - Search epochs (in the same run);   \subref{fig:CIFAR10TestAccuracy-SearchEpochSkip-Conv}: CIFAR10: Accuracy - Search epochs (multiple runs);  \subref{fig:CIFAR10 Test Accuracy – Search Iteration (Skip-Conv)}: CIFAR10: Accuracy - Search iterations;    \subref{fig:CIFAR10_1_vs_2_Iteration_Search_Accuracy_Comparison}: Accuracy on CIFAR10: 1 vs 2 search iterations (multiple runs).}
	\label{fig:Search and evaluation period and candidate operators' impacts to FTSO on CIFAR 1}
\end{figure}

FTSO's huge reduction on the parameter numbers provides us a large number of benefits. As mentioned above, it allows the algorithm to converge in only a few iterations and prevents over-fitting. This is because when extracting the discrete sub-graph from the super-net, many architecture parameters are set to 0. The introduced disturbance impacts more on the over-fitting super-nets since they prefer sharper local minimums. 
Furthermore, Since FTSO only contains one operator with 0 parameters, the Matthew effect is eliminated.

\section{Experiments}
\label{section: Experiments}

Our search algorithm is evaluated on the three most widely-used datasets in NAS papers, namely, CIFAR10~\citep{CIFAR10}, ImageNet~\citep{ImageNet} and NATS-Bench~\citep{NATS-Bench}. Following DARTS, our search space contains a total of eight candidate operators: $3\times3$ and $5\times5$ \textit{separable convolutions}, $3\times3$ and $5\times5$ \textit{dilated separable convolutions}, $3\times3$ \textit{max} and \textit{average pooling}, \textit{skip connection} (i.e., $output=input$) and \textit{zero} (i.e., $output=0$). When searching for the topology, we pick only one operator from the candidate set. As mentioned in Section \ref{section: FTSO: Effective NAS via First Topology Second Operator}, we have two strategies to determine the operators, including the one based on gradients and the one directly replacing operators. In Sections \ref{section: experiments on cifar10} and \ref{section: experiments on imagenet}, we focus more on the second configuration. In Section \ref{section: ablation study}, we have a comprehensive comparison on these two strategies. All detailed configurations are shown in the Appendices. 
In addition, most of our experiments only search for one epoch or one iteration because of the benefits of FTSO’s huge reduction on the parameter numbers. For more experimental support, please refer to Section \ref{section: ablation study}. Note that it is almost impossible for the existing models to obtain satisfying results via only searching for one epoch or one iteration because their super-nets contain a large amount of parameters, which require a long period to tune.

\begin{figure}[htbp]
	\centering
	\vspace{-1ex}
    \subfigure[]{
		\begin{minipage}[b]{0.22\textwidth}
			\includegraphics[width=1\textwidth]{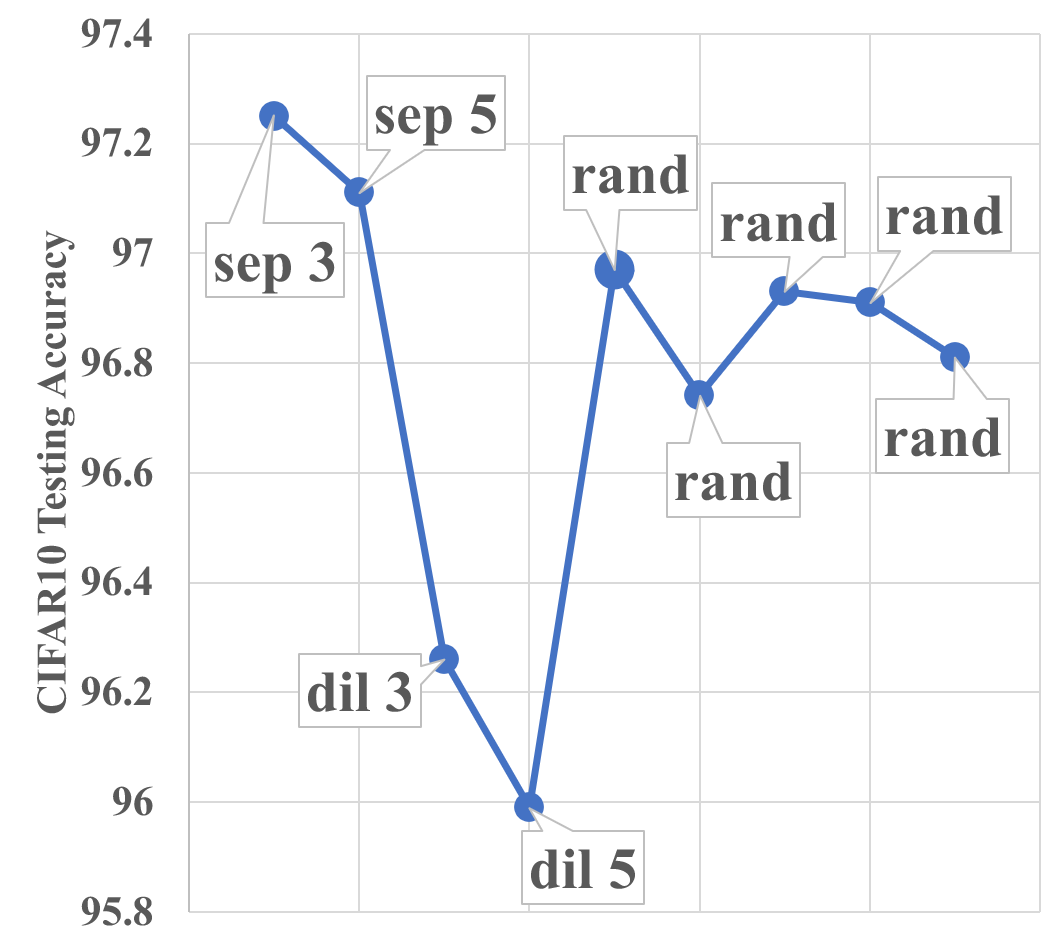} 
		\end{minipage}
		\label{fig:CIFAR10_Test_Accuracy_-_Replace_Op__Skip-Others_}
	}
    	\subfigure[]{
    		\begin{minipage}[b]{0.22\textwidth}
		 	\includegraphics[width=1\textwidth]{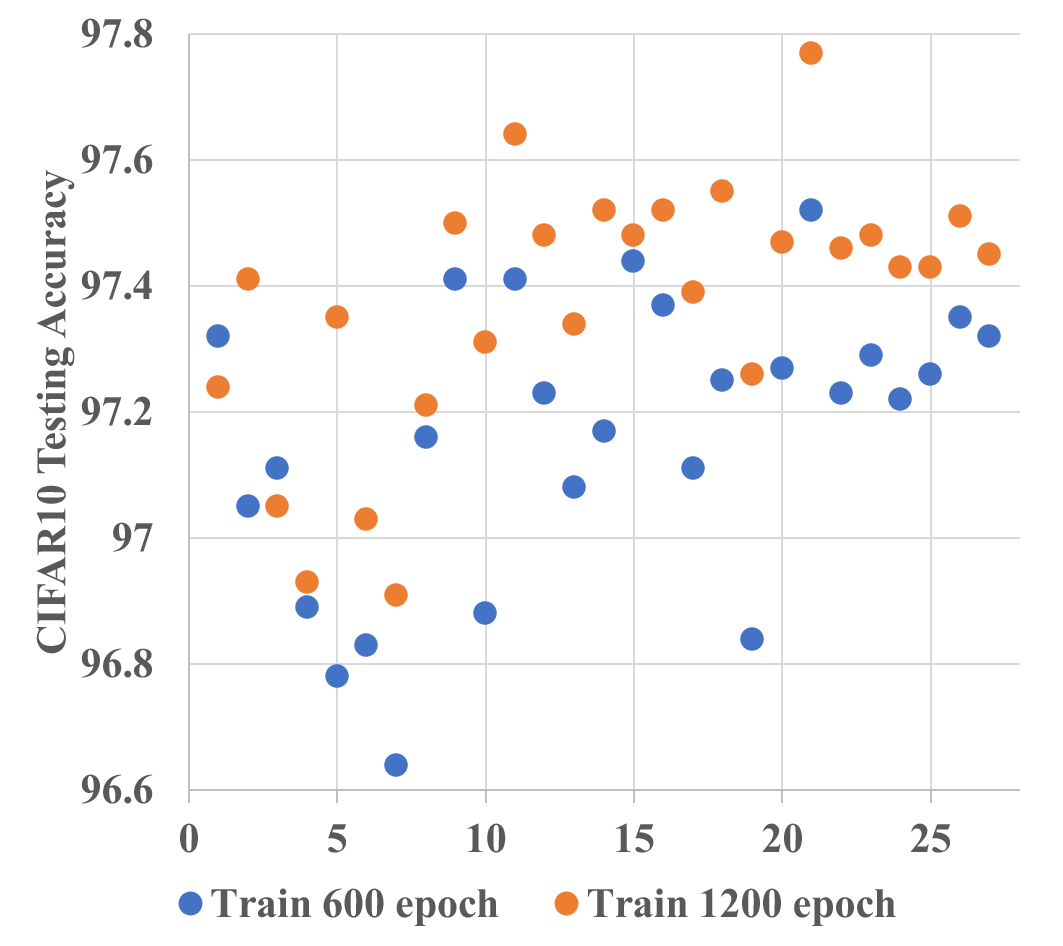}
    		\end{minipage}
		\label{fig:CIFAR10TestAccuracy-TrainEpoch}
    	}
    \subfigure[]{
        \begin{minipage}[b]{0.22\textwidth}
       		 \includegraphics[width=1\textwidth]{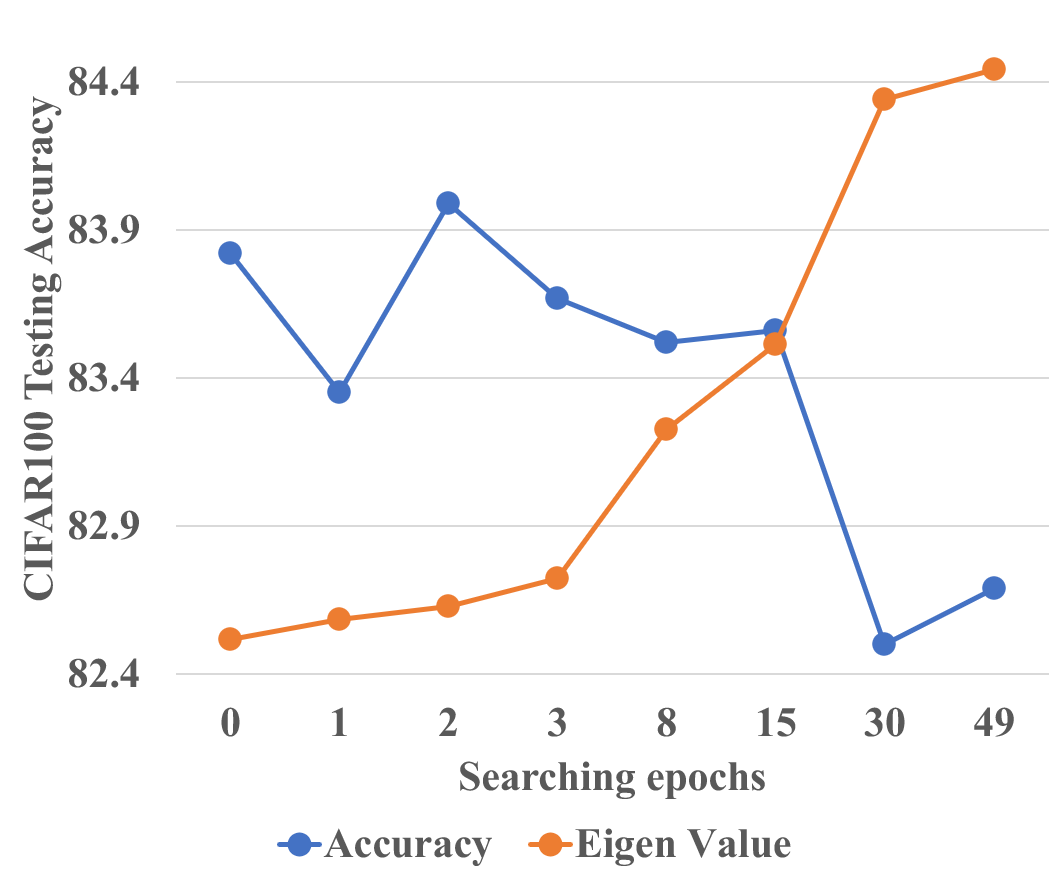}
        		\end{minipage}
    		\label{fig:Epoch-wiseCIFAR100TestAccuracyMaxEigenValue}
        	}
        \subfigure[]{
        	\begin{minipage}[b]{0.22\textwidth}
       		 \includegraphics[width=1\textwidth]{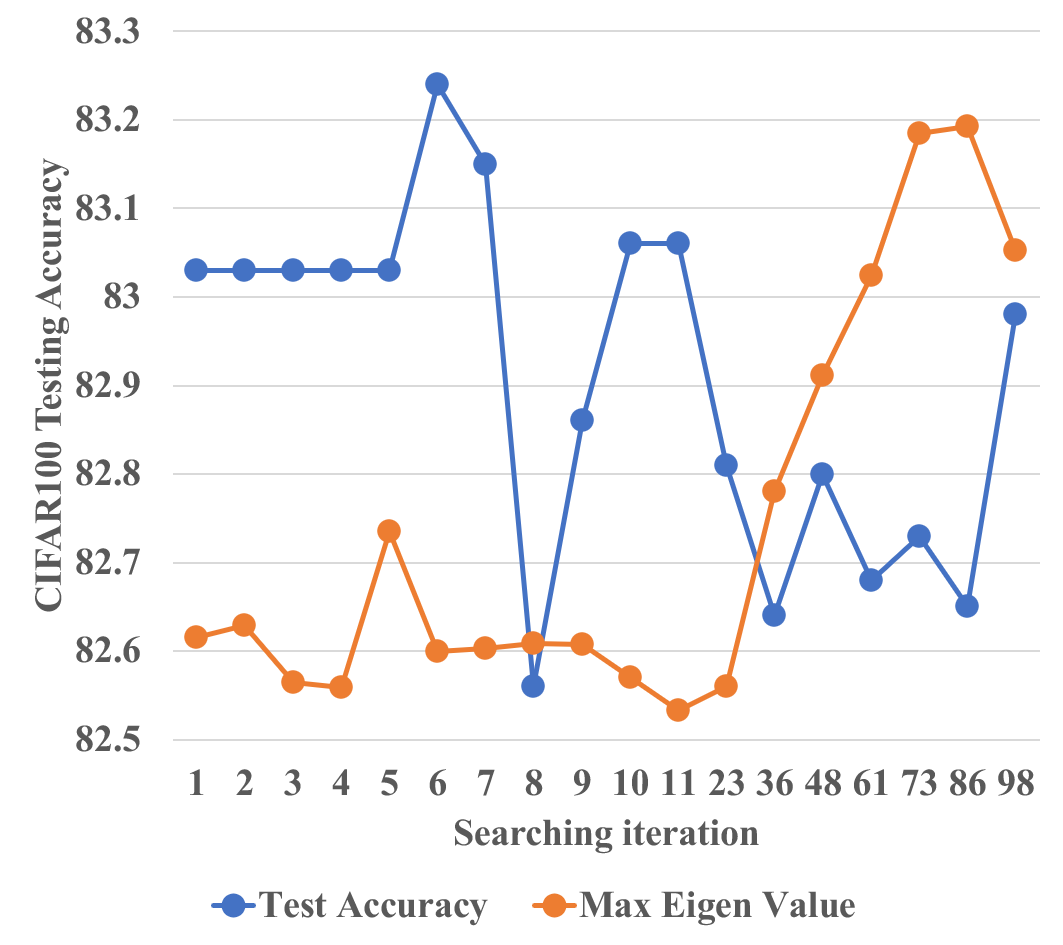}
        		\end{minipage}
    		\label{fig:Iter-wiseCIFAR100TestAccuracy-MaxEigenValue}
        	}
    \vspace{-2ex}
	\caption{Ablation study Part 2.   \subref{fig:CIFAR10_Test_Accuracy_-_Replace_Op__Skip-Others_}: CIFAR10: Accuracy - Operator replacing skip-connections; \subref{fig:CIFAR10TestAccuracy-TrainEpoch}: CIFAR10: Accuracy - Training epochs of the found architecture;  \subref{fig:Epoch-wiseCIFAR100TestAccuracyMaxEigenValue}: Epoch-wise on CIFAR100: Accuracy - Max eigenvalue $\nabla^2_{Arch}\mathcal{L}_{val}$; \subref{fig:Iter-wiseCIFAR100TestAccuracy-MaxEigenValue}: Iteration-wise on CIFAR100: Accuracy - Max eigenvalue $\nabla^2_{Arch}\mathcal{L}_{val}$.}
	\vspace{-2ex}
	\label{fig:Search and evaluation period and candidate operators' impacts to FTSO on CIFAR 2}
\end{figure}

\subsection{Results on CIFAR10}
\label{section: experiments on cifar10}

We compare FTSO to existing state-of-the-art NAS methods in Table \ref{Comp_CIFAR10}. In the experiment, we only search for the topology with skip-connections and then replace them all with \textit{$3\times3$ separable convolutions}. The reason that we do not adopt \textit{$5\times5$ separable convolutions} is that the pre-processed input images do not have enough resolutions, and the network is rather deep. After a few layers, the convolution's receptive field becomes larger than the whole image. At that time, larger convolutional kernels may not bring benefits. Instead, the extra parameters brought by the larger kernel size may lead to over-fitting. 

On the other hand, suppose both the image's resolution and the dataset's scale are big enough and the evaluation period is adequate,  the \textit{$5\times5$ separable convolution} might be a better choice. 
The found architecture after one epoch's search is shown in Figure \ref{FTSO's found architecture CIFAR10}. Due to FTSO containing only a few trainable parameters, it can even achieve comparable accuracy to PC-DARTS with only a one-time gradient update. Under this configuration, a mere $0.68$ seconds are required and $99.993\%$ of the search time is saved. In addition, as shown in Table \ref{Table:: Ablation study on FTSO}, when the topology is searched with powerful operators for a long period, an additional operator search usually helps. However, when we search for the topology with simple operators for a short period, omitting the operator search may lead to better results. This is because with simple operators and very few updates, the found topology can already generalize quite well.

\subsection{Results on ImageNet}
\label{section: experiments on imagenet}
On ImageNet, we use similar configurations to those on CIFAR10. When searching, we have two configurations. The detailed configurations are shown in the Appendices. Our experiments in Table~\ref{Comp_CIFAR10} show that FTSO is significantly superior to existing methods in both efficiency and effectiveness. The found architectures after one epoch's search on CIFAR10 and the entire ImageNet are shown in Figure \ref{FTSO's found architecture CIFAR10}. 

It is surprising that the best architecture we found on ImageNet is the shallowest and widest one. Compared to the much more 'reasonable' architectures shown in Figure \ref{fig:Normal Cell CIFAR10 sep + op} and \ref{fig:Reduction Cell CIFAR10 sep + op}, which were found with the topology search only containing $3\times3$ separable convolutions and an additional operator search on CIFAR10, the 'abnormal' architecture, containing the same amount of FLOPs and parameters, can achieve $0.78\%$ higher testing accuracy. We think this is because the whole model is stacked with many cells. If the depth of each cell is too high, it leads to a very deep neural network. At that time, because all the operators in our found architecture are convolutions, we cannot use skip connections to facilitate gradients’ propagation in ResNet’s manner. In this way, both the vanishing and explosion of gradients may prevent the deeper models from higher performance.

\subsection{Results on NATS-Bench}
\label{section: experiments on NATS-Bench}
In the search space of NATS-Bench, there is one input node, three intermediate nodes and one output node, and each intermediate node connects to all its predecessors. Here we implement FTSO based on DARTS instead of PC-DARTS, and we compare FTSO's performance to other NAS algorithms in Table \ref{Comp_NATSBench}. It is shown that FTSO dominates DARTS in all configurations. Coinciding with our analysis in Section \ref{section: FTSO: Effective NAS via First Topology Second Operator}, the architectures found by DARTS tend to only contain simple operators, thus cannot achieve satisfying accuracy. For example, when searching on CIFAR10, the architecture found by DARTS is full of \textit{skip connections} as shown in Figure \ref{fig:NASBenchallskip}. By comparison, as shown in Figure \ref{fig:NASBenchallsep}, the architecture found by FTSO is much more powerful.

\subsection{Ablation Study}
\label{section: ablation study}
In terms of a topology-only search, one epoch is just enough, thanks to the many fewer kernel weights contained in FTSO, and more search epochs bring obvious disadvantages because of over-fitting. Since one epoch performs better than more epochs, it raises the question whether one iteration is also superior to more iterations. We find that the found architecture's performance generally first drops and then increases in the first epoch, and then always decreases after the first epoch. In Figure \ref{fig:CIFAR10 Test Accuracy – Search Iteration (Skip-Conv)} we show that although one iteration cannot surpass one epoch, it is better than a few iterations. This is because when we only search for one iteration, the model does not over-fit the data, thus the model generalizes well. When we only searched for a few iterations, the number of different images seen by the model is not big enough. However, since the super-net only contains skip connections, such a number of gradient updates is enough for the architecture parameters to become over-fitted. This is the reason that a few iterations perform worse than one iteration. After we have searched for one whole epoch, the super-net has seen enormous different images, which helps it to generalize better on the testing set. This is the reason one epoch performs the best. 
In terms of whether we should search for one iteration or two, in Figure \ref{fig:CIFAR10_1_vs_2_Iteration_Search_Accuracy_Comparison}, we show that both choices work well. When we do not search for the operators after the topology search, we assign all the remaining edges a fixed operator. Thus, which operator we should choose becomes a critical question. Figure \ref{fig:CIFAR10_Test_Accuracy_-_Replace_Op__Skip-Others_} show that a \textit{$3\times3$ separable convolution} can indeed outperform all other operators in terms of accuracy. 

\begin{figure}[H]
	\centering
	\vspace{-2ex}
	
    \subfigure[]{
		\begin{minipage}[b]{0.22\textwidth}
			\includegraphics[width=1\textwidth]{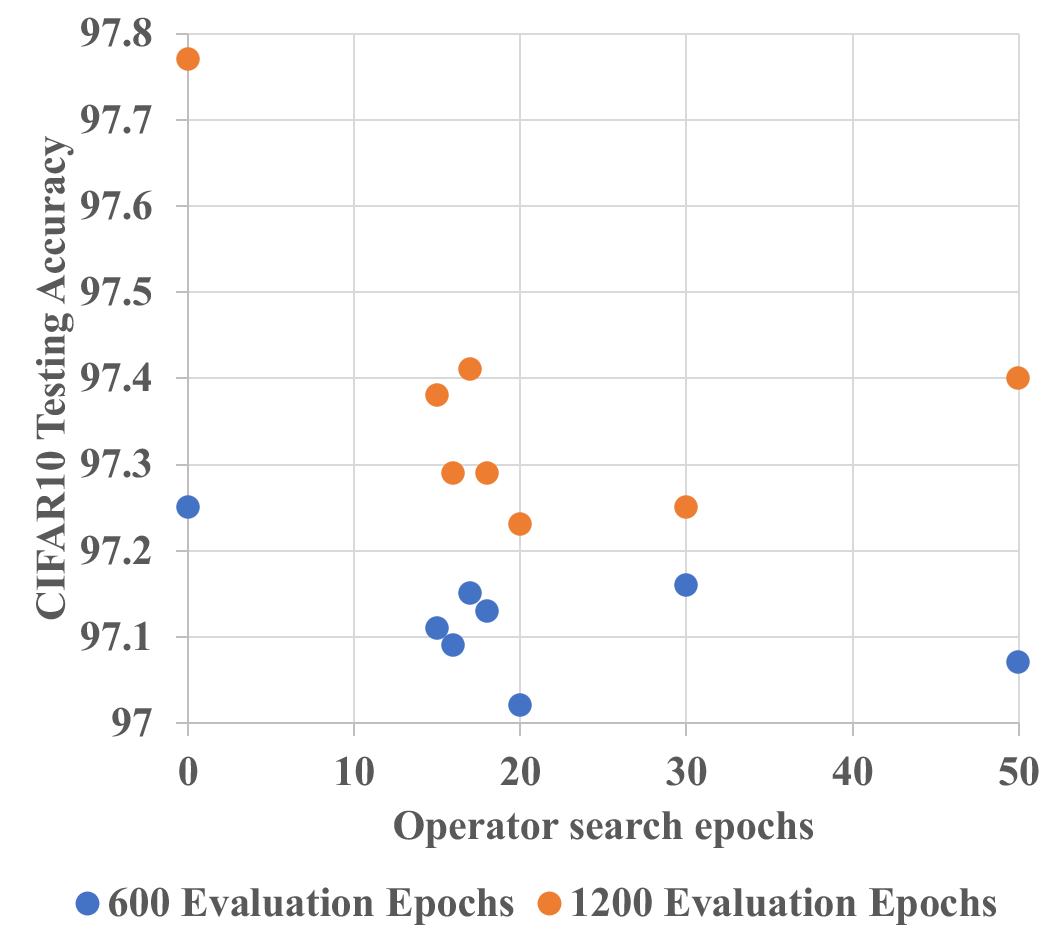} 
		\end{minipage}
		\label{fig:Operator search's Impacts by epochs (Same Run)}
	}
    	\subfigure[]{
    		\begin{minipage}[b]{0.22\textwidth}
   		 	\includegraphics[width=1\textwidth]{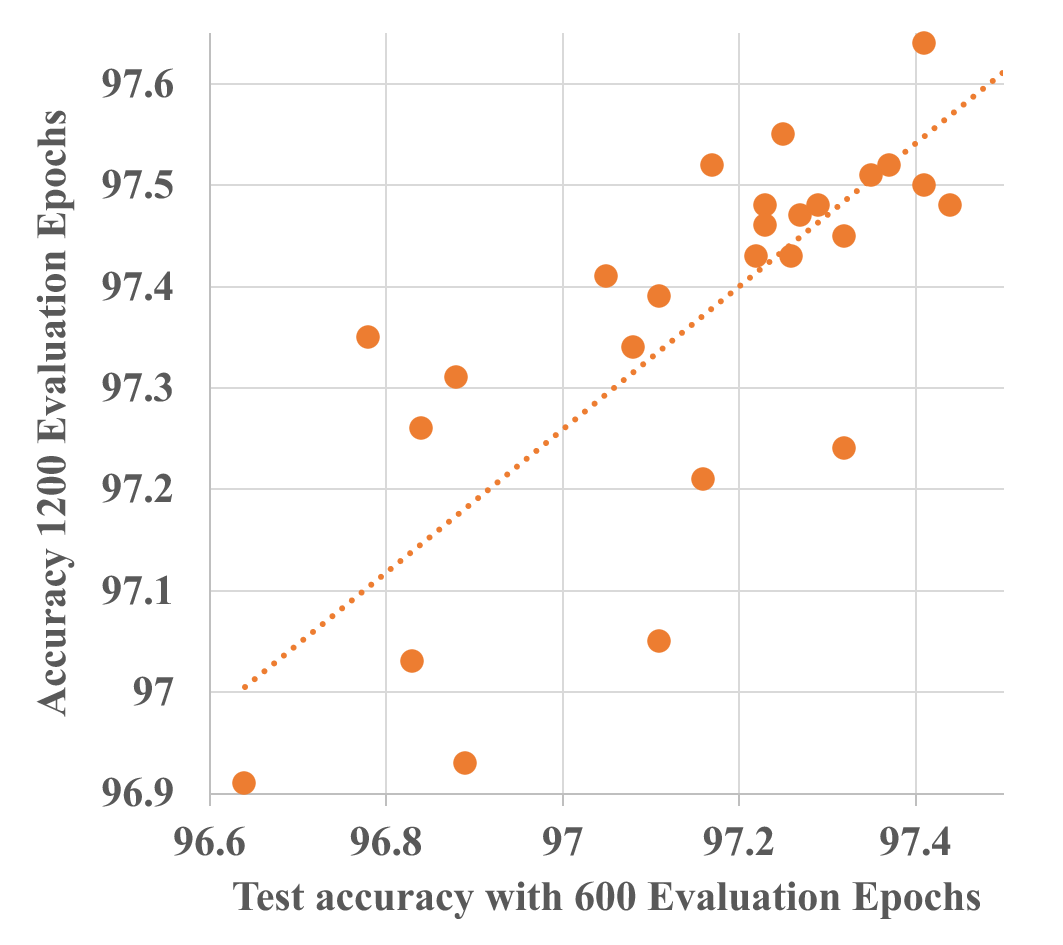}
    		\end{minipage}
		\label{fig:EvaluationEpochsImpactsAbsoluteTestAccuracy}
    	}
	\subfigure[]{
		\begin{minipage}[b]{0.22\textwidth}
			\includegraphics[width=1\textwidth]{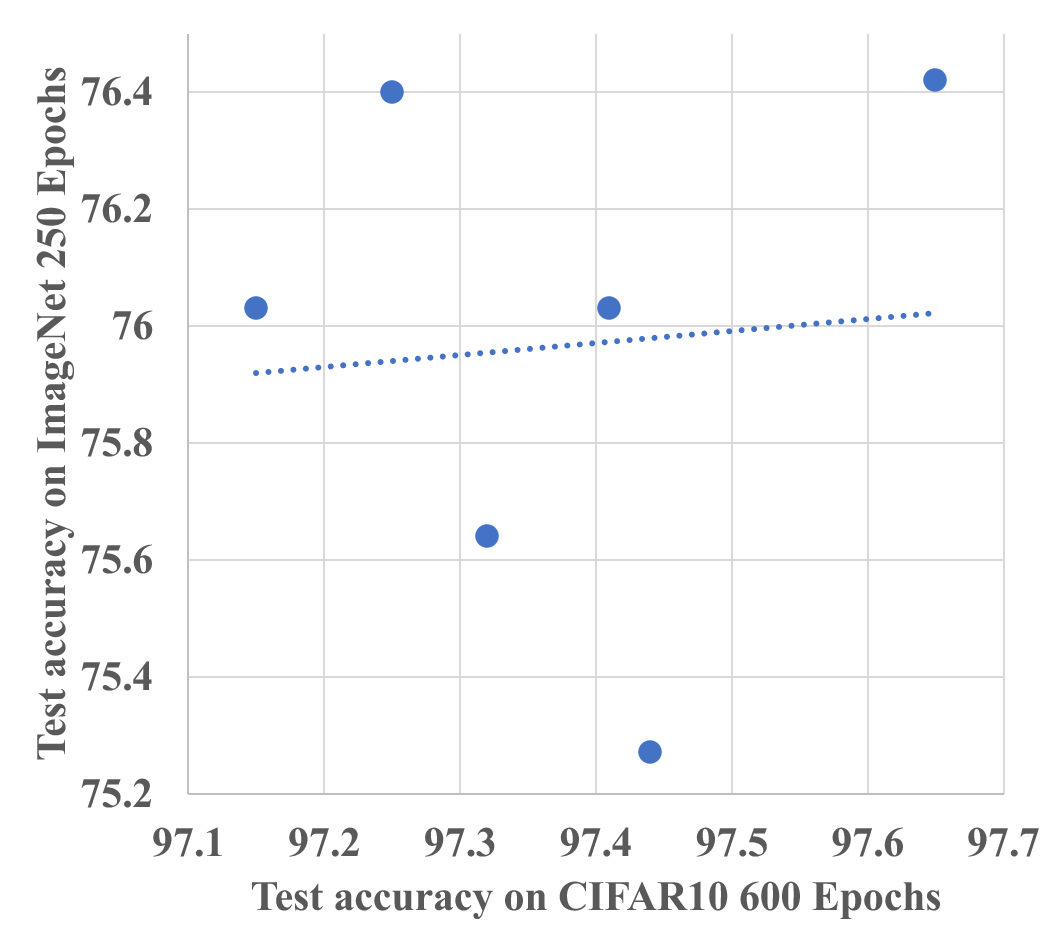} 
		\end{minipage}
		\label{fig:TestAccuracyRelationshipbetweenCIFAR10600epochsandImageNet250epochs}
	}
    	\subfigure[]{
    		\begin{minipage}[b]{0.22\textwidth}
		 	\includegraphics[width=1\textwidth]{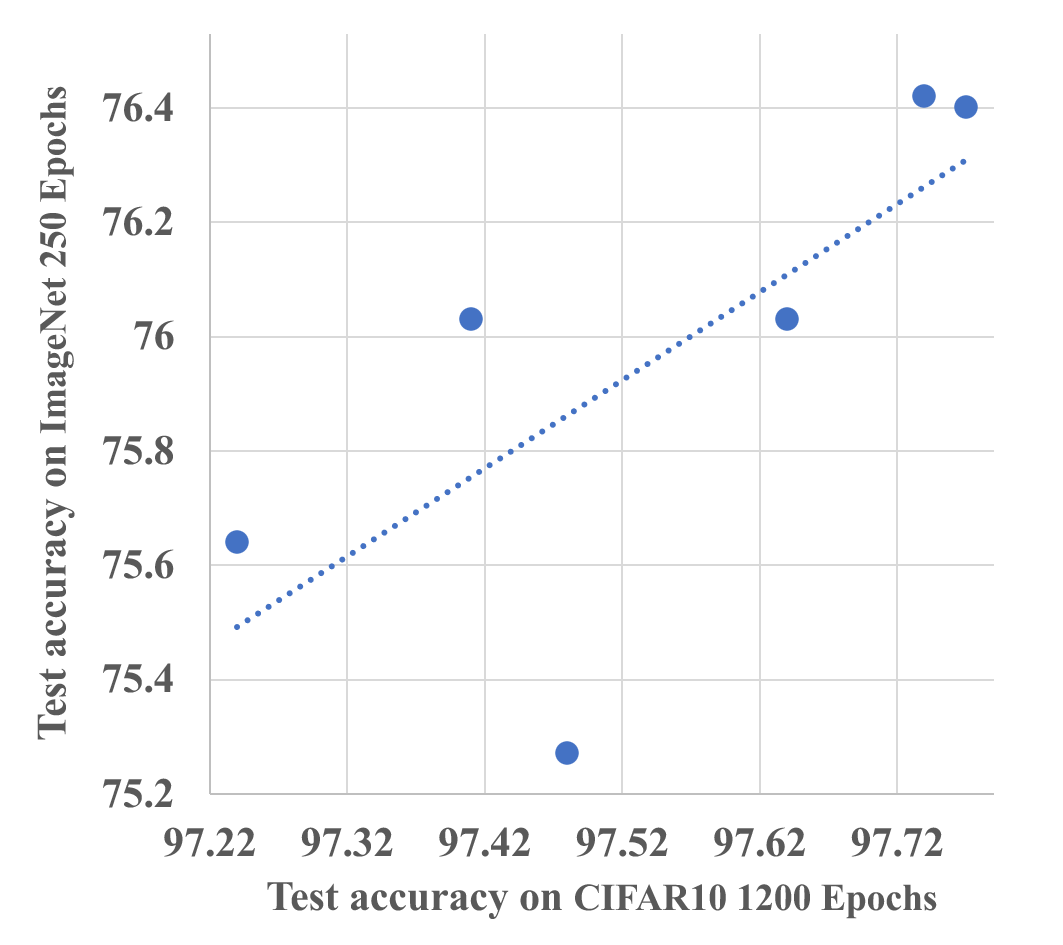}
    		\end{minipage}
    	\label{fig:TestAccuracyRelationshipbetweenCIFAR101200epochsandImageNet250epochs}
    	}
    \vspace{-1ex}
    \caption{Ablation study Part 3.   \subref{fig:Operator search's Impacts by epochs (Same Run)}: CIFAR10: Operator search's epoch-wise impact on testing accuracy (0 epoch means only searching for the topology);    \subref{fig:Evaluation Epochs’ Impacts to Test Accuracy}: CIFAR10: The found architecture's training epochs’ impacts on the testing accuracy (same architecture);   \subref{fig:TestAccuracyRelationshipbetweenCIFAR10600epochsandImageNet250epochs}:Testing accuracy: CIFAR10 600 evaluation epochs - ImageNet 250 evaluation epochs (evaluating after training the same architecture on different datasets);   \subref{fig:TestAccuracyRelationshipbetweenCIFAR101200epochsandImageNet250epochs}: Testing accuracy: CIFAR10 1200 evaluation epochs - ImageNet 250 evaluation epochs (evaluating after training the same architecture on different datasets)}
	\vspace{-2ex}
	\label{fig:Search and evaluation period and candidate operators' impacts to FTSO on CIFAR Table 3}
\end{figure}

As shown in Figure \ref{fig:CIFAR10TestAccuracy-TrainEpoch}, we find that under a different number of evaluation epochs, both the absolute value and the relative ranking of the same architecture's testing accuracy may vary; i.e., some architectures which perform well within 600 epochs perform poorly after 1200 epochs. However, in general, they still obey a positive correlation, with a Pearson correlation coefficient of $0.77$, as shown in Figure \ref{fig:Evaluation Epochs’ Impacts to Test Accuracy}. In terms of the generalization ability from CIFAR10 to ImageNet, Figure \ref{fig:TestAccuracyRelationshipbetweenCIFAR101200epochsandImageNet250epochs} reveals that the architectures which perform well after long-term evaluation on CIFAR10 can usually generalize better on ImageNet, with a correlation coefficient of $0.7$; yet, as shown in Figure \ref{fig:TestAccuracyRelationshipbetweenCIFAR10600epochsandImageNet250epochs}, there is no guarantee that those working well on CIFAR10 within limited evaluation epochs can also dominate on ImageNet. This is because it only proves that they can converge quickly, not that they can converge to a global optimum.

In Figures \ref{fig:Epoch-wiseCIFAR100TestAccuracyMaxEigenValue} and \ref{fig:Iter-wiseCIFAR100TestAccuracy-MaxEigenValue}, we show that one epoch is not only an optimal choice on CIFAR10, but also enough for the topology-only search on CIFAR100. In addition, as the search epochs and iterations increase, the max eigenvalue of the loss's Hessian matrix on the validation set increases. At the same time, the testing accuracy generally decreases because the model's generalization ability is dropping. This phenomenon is particularly obvious epoch-wise because, after just a few iterations, the model can already reach a comparable accuracy on the training set. Then the model's performance on the testing set starts to relate with its generalization ability.

\section{Generalization to other tasks and search spaces}
\label{section: Generalization to other tasks and search spaces}
As shown in Section \ref{section: Experiments}, FTSO works well under different search spaces and node numbers. Theoretically, FTSO's advantages to DARTS can be enlarged while the search space and node number increase. The reason is that FTSO reduces the computational cost from $O(n^3)$ to $O(n^2)$ and avoids over-fitting. Based on such considerations, as the future work, we plan to apply FTSO in high-level tasks, for example, instance segmentation and multi-view stereo. Although in this paper, we establish FTSO within differentiable search spaces, in fact, the first topology second operator strategy is not limited to any specific search space or tasks. Whether or not the search space is discrete or continuous, or the search algorithm is gradient-based or reinforcement learning-based, we first shrink the candidate operator set, and only retain the simplest operator, in language modeling which might be a \textit{skip connection} or a \textit{pooling layer}. After this, the size of the whole search space is reduced in magnitude. Then, we search for the best topology with any available search algorithm. In this way, a promising topology can be found. Then, we can either directly assign each edge a powerful operator, in language modeling which might be a \textit{LSTM unit} or an \textit{attention layer}, or use gradients to search for operators. Generally, the directly replacing strategy leads to higher accuracy, and the gradient-based strategy reduces the model complexity.

\section{Conclusion}
In this paper, we propose an ultra computationally efficient neural architecture search method named FTSO, which reduces NAS's search time cost from days to less than 0.68 seconds, while achieving 1.5\% and 0.27\% accuracy improvement on ImageNet and CIFAR10, respectively. Our key idea is to divide the search procedure into two sub-phases. In the first phase, we only search for the network's topology with simple operators. Then, in the second phase, we fix the topology and only consider which operators we should choose. 

Our strategy is concise in both theory and implementation, and our promising experimental results show that current NAS methods contain too much redundancy, which heavily impacts the efficiency and becomes a barrier to higher accuracy. What is more, as mentioned in Section \ref{section: Generalization to other tasks and search spaces}, our method is not bound by differentiable search spaces as it can also cooperate well with existing NAS approaches.

\bibliography{example_paper}
\bibliographystyle{icml2021}

\end{document}